\documentclass[]{acmart}
\AtBeginDocument{%
  }

\setcopyright{acmlicensed}
\copyrightyear{2018}
\acmYear{2018}
\acmDOI{XXXXXXX.XXXXXXX}
\acmISBN{978-1-4503-XXXX-X/2018/06}

\usepackage{tikz}
\usepackage{graphicx}
\usepackage{bbm}
\usepackage{makecell}

\newcommand{\paren}[1]{\left( #1 \right)}

\newcommand{\norm}[1]{\left\lVert#1\right\rVert}

\usepackage{enumitem}

\newtheorem{prop}{Proposition}
\newtheorem{definition}{Definition}

\usepackage{amsmath}
\usepackage{booktabs}
\usepackage{amsfonts}

\begin{document}

\title{CoRe: Coherency Regularization for Hierarchical Time Series}

\author{Rares Cristian}
\email{raresc@mit.edu}
\affiliation{%
  \institution{MIT}
  \country{U.S.}
}
\author{Pavithra Harsha}
\email{pharsha@us.ibm.com}
\affiliation{%
  \institution{IBM Research}
  \city{Yorktown Heights}
  \state{New York}
   \country{U.S.}
}
\author{Georgia Perakis}
\email{georgiap@mit.edu}
\affiliation{%
  \institution{MIT}
  \city{Cambridge}
  \state{Massachusetts}
   \country{U.S.}
}
\author{Brian Quanz}
\email{blquanz@us.ibm.com}
\affiliation{%
  \institution{IBM Research}
  \city{Yorktown Heights}
  \state{New York}
   \country{U.S.}
}

\renewcommand{\shortauthors}{Cristian et al.}

\begin{abstract}
Hierarchical time series forecasting presents unique challenges, particularly when dealing with noisy data that may not perfectly adhere to aggregation constraints. This paper introduces a novel approach to soft coherency in hierarchical time series forecasting using neural networks. We present a network coherency regularization method, which we denote as CoRe (Coherency Regularization), a technique that trains neural networks to produce forecasts that are inherently coherent across hierarchies, without strictly enforcing aggregation constraints. Our method offers several key advantages. (1) It provides theoretical guarantees on the coherency of forecasts, even for out-of-sample data. (2) It is adaptable to scenarios where data may contain errors or missing values, making it more robust than strict coherency methods. (3) It can be easily integrated into existing neural network architectures for time series forecasting. We demonstrate the effectiveness of our approach on multiple benchmark datasets, comparing it against state-of-the-art methods in both coherent and noisy data scenarios. Additionally, our method can be used within existing generative probabilistic forecasting frameworks to generate coherent probabilistic forecasts. Our results show improved generalization and forecast accuracy, particularly in the presence of data inconsistencies. On a variety of datasets, including both strictly hierarchically coherent and noisy data, our training method has either equal or better accuracy at all levels of the hierarchy while being strictly more coherent out-of-sample than existing soft-coherency methods.
\end{abstract}

\ccsdesc[500]{Computing methodologies~Neural networks}

\keywords{Hierarchical time series forecasting, neural networks, regularization}


\maketitle

\section{Introduction}

In many practical multivariate time series forecasting applications, series have a natural hierarchical nature where upper levels of the hierarchy are aggregates of lower levels. See for instance \cite{hyndman2018forecasting} (chapter 10) for an introduction to hierarchical forecasting.  For example, in many cases, time series can be aggregated by space, using features such as geographic region,  sales of retail products across stores \cite{petropoulos2022forecasting},  employment levels across regions \cite{cuturi2011fast}, traffic patterns \cite{abs2020labourforce} and more. Figure \ref{fig:hierarchy-example} shows a toy example of such a hierarchy. Sometimes multiple-hierarchies may also co-exist, such as product category, business, as well as location hierarchies. 

Forecasts at all levels are valuable and, moreover, understanding patterns at higher levels can help in improving prediction at lower levels. Vice versa, lower level information can also help make better predictions at higher levels, allowing one to factor in more local information. Each level of the hierarchy often contains data that behaves quite differently. For instance, lower levels are often sparse or noisy, while higher levels have more stable behavior over time. As such, recent methods like \cite{han2021simultaneously},  \cite{kamarthi2022profhit}, \cite{rangapuram2021end}, \cite{tsiourvaslearning} have focused on creating \emph{coherent} forecasts that predict all levels simultaneously while also satisfying aggregation constraints of the hierarchy. 

However, while many datasets may have a hierarchical structure, they may not  exactly satisfy the aggregation constraints. These deviations from exact aggregation can result from measurement or reporting errors, for example. Therefore, strictly enforcing coherency in forecasts could negatively affect accuracy in these cases. As an example, this problem has been addressed by a recent paper \cite{kamarthi2022profhit} which instead enforces a \emph{soft} coherency of the model forecasts. The authors propose a soft distributional consistency method which provides distributional forecasts of each series. They enforce soft coherency by penalizing the network on its training data by how much its forecasts deviate from exact coherency. In this way, they can trade off between accuracy and coherency. However, such an approach does not guarantee any coherency properties out-of-sample. In this paper we present the following contributions.

\begin{enumerate}
    \item We present a soft coherency method (CoRe) based on regularizing the layers of the neural network itself. We first introduce  \emph{network coherency regularization} to train any general class of neural networks so that they universally satisfy (soft) coherence for any input. The framework is described in section \ref{sec:coherency} and illustrated in figure \ref{fig:arch}.
    \item We prove that the network coherency regularization bounds how closely the output of the network satisfies coherency. We prove this with high probability for any input.
    \item We discuss how our approach can be integrated with existing distributional forecasting methods to ensure (softly) coherent distributions. Specifically, ensure that each sample of the distribution is nearly coherent. For instance, our approach can impose soft coherency on samples  generated by Variational Autoencoder (VAE) methods for distributional forecasting. 
    \item Computationally, we test on a variety of hierarchical time series problems as well as data with noise which requires soft coherency. We compare against state of the art models like \cite{rangapuram2021end}, \cite{kamarthi2022profhit} and show improved generalization out-of-sample both for point and distributional forecasts. 
\end{enumerate}


\noindent A key contribution of our method is in ensuring that each sample of the forecasted distribution is (nearly) coherent. This is in contrast with a soft coherency method like \cite{kamarthi2022profhit} which predicts a distribution (like a Gaussian distribution) independently for each series. This limits the coherency of their method since sampling from each distribution independently will not produce a highly coherent sample across all series. We discuss this further in Section \ref{sec:dist-forecast}.

Finally, an additional benefit of our approach is that we guarantee soft coherency on any input data, including out-of-sample and even out-of-distribution, since coherence is an inherent property of the final model learned. This is unlike other soft coherency methods which only guarantee coherence on training data, our method can ensure soft coherency on data not observed during training (our-of-sample) or coming from a changed or different distribution altogether (out-of-distribution). Therefore, our approach could be applied to out-of-distribution generalization such as \cite{liu2024time} in order to additionally ensure soft coherency.


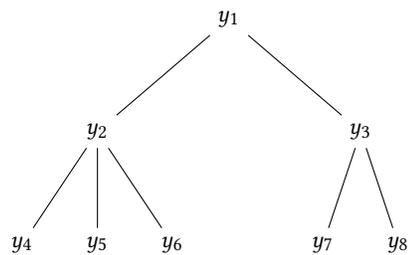
\begin{figure}[t]
    \centering
    \begin{tikzpicture}[level 1/.style={sibling distance=3.5cm},
      level 2/.style={sibling distance=1cm},
      level distance=1.5cm]
      \node {$y_1$}
        child {node {$y_2$}
          child {node {$y_4$}}
          child {node {$y_5$}}
          child {node {$y_6$}}
        }
        child {node {$y_3$}
          child {node {$y_7$}}
          child {node {$y_8$}}
        };
    \end{tikzpicture}
    \caption{An example time series hierarchy with three levels and eight series. For ex., the series $y_1$ at the top level is an aggregation of the lower-level series $y_2, y_3$.}
    \label{fig:hierarchy-example}
\end{figure}

\subsection{Related literature}

\paragraph{Classical Time Series Forecasting}

Traditional time series forecasting models, predating deep learning models, have wide-spread applications and continue to be used in industries today. The forecasting models include exponential smoothing \cite{mckenzie1984general}, Holt-Winters \cite{chatfield1978holt}, and Autoregressive Integrated Moving Average (ARIMA) \cite{box1968some}. These models can produce accurate predictions of future data using only a few key statistical features of the training data. However, neural network models (to name a few, \cite{lim2021temporal}, \cite{zhou2021informer}, \cite{wu2021autoformer}) can make significant improvements over these classical models, especially in large datasets, by capturing more complex patterns in the data.  Recurrent neural networks \cite{grossberg2013recurrent} effectively model time series, and sequential data in general, by having sequential states based on previous timesteps. 
Recently, transformers \cite{vaswani2017attention} and their variants as well as large language models like ChatGPT \cite{wu2023brief} have also been applied to time series forecasting \citep{wen2022transformers,jin2023time,ekambaram2024ttms}, improving performance with self-attention. 

\paragraph{Hierarchical Time Series Forecasting}

Early work in hierarchical time series forecasting involved \textit{top-down} or \textit{bottom-up} approaches, where models would predict the highest level series and disaggregate to retrieve lower level (top-down) predictions \cite{gross1990disaggregation} or predict the lowest levels series and aggregate (bottom-up) \cite{orcutt1968aggregation}. By only predicting on a subset of the hierarchy, these methods didn't fully utilize all information present in the hierarchical data. \citet{hyndman2011optimal} introduced a model that predicts all levels of the hierarchy independently and use regression to \textit{reconcile} those forecasts, making them coherent with the hierarchy. This outperformed the top-down and bottom-up approaches. One disadvantage with this predict-then-reconcile approach, though, is that information is lost in the separation of the prediction and reconciliation steps. Other approaches such as by \citet{van2015game} use game theory to compute reconciliations for forecasts, while MINT \cite{wickramasuriya2024probabilistic}  solves an optimization problem based on the coherency constraints to produce reconciled forecasts.

\citet{han2021simultaneously} use simultaneous forecasting and reconciliation, adding the reconciliation of all series at once into the training step of the models by regularizing a neural network with the coherency of quantiles. \citet{rangapuram2021end}, \citet{tsiourvaslearning} extend this by making the prediction and reconciliation process end-to-end differentiable, producing a single neural network that can predict and reconcile simultaneously. To make the reconciliation process differentiable, they project predictions into coherent space. Since their projection can be computed using a single matrix multiplication, the authors show it is differentiable and can be incorporated into a neural network. They add projection on top of DeepVAR \cite{salinas2020deepar}, which uses an auto-regressive recurrent network model to predict probabilistic forecasts. Thus, the final process is to predict a distributional mean and variance from DeepVAR, sample from that prediction, then project the sampled outputs into coherent space. 

PROFHiT \cite{kamarthi2022profhit} also uses probabilistic prediction and end-to-end differentiable reconciliation but also introduces the concept of a soft distributional consistency regularization, which is closer to our network coherency regularization. They define soft distributional consistency regularization as the distance between each prediction at the given time-step and its given hierarchy. There are a few differences between their work and ours. 

Our approach can be used to generate a non-parametric distributional forecast (such as generate samples from a VAE-based approach) while their method forecasts mean and variance of gaussian distribution for each series. However, their approach makes it more difficult to properly model coherency --- generating a sample from each distribution for each series will not result in a coherent sample overall since each Gaussian distribution for each series is independent. Our approach on the other hand ensures that each generated sample is (softly) coherent.  In addition, our approach also ensure soft coherence on any possible input, including on out-of-sample and out-of-distribution data. This is because we enforce coherency as an inherent property of the neural network model itself. On the other hand, \cite{kamarthi2022profhit} only applies a coherency loss to a network's output on training data, and so can have more difficulty generalizing out of sample as a result.



\section{Coherency framework}
\label{sec:coherency}

We first formally define coherency. Next, we introduce our method to train models to satisfy near-coherency by construction. We are given a set of time series $y_1, \dots, y_m$, with each $y_i$ consisting of $S$ time points and observations $y_i^1, \dots, y_i^S$. Similarly, denote a forecast for series $i$ as $\hat{y}_i$. A forecast is \textit{coherent} if it obeys the aggregation constraints of the hierarchy.  For example, a forecast of the hierarchy shown in Figure \ref{fig:hierarchy-example} would be coherent if $\hat{y}_2 = \hat{y}_4 + \hat{y}_5 + \hat{y}_6$, $\hat{y}_3 = \hat{y}_7 + \hat{y}_8$, and $\hat{y}_1 = \hat{y}_4 + \hat{y}_5 + \hat{y}_6  + \hat{y}_7 + \hat{y}_8$. In this paper, we encode coherency with respect to the bottom level of the hierarchy, so for $\hat{y}_1$ we aggregate the lowest levels of the series, instead of directly using $\hat{y}_2$ and $\hat{y}_3$. 

We can encode the hierarchy of a multivariate time series in an aggregation matrix, which we call $A$. For a pair of series $i,j$, we set $A_{i, j} = 1$ if $y_j$ is a base series (at the lowest level, or leaf of the tree) that is part of the aggregate for $y_i$. For our example from Figure \ref{fig:hierarchy-example}, we would set 
\begin{equation*}
    A = \begin{bmatrix}
    0 & 0 & 0 & 1 & 1 & 1 & 1 & 1 \\ 
    0 & 0 & 0 & 1 & 1 & 1 & 0 & 0 \\ 
    0 & 0 & 0 & 0 & 0 & 0 & 1 & 1 \\
    0 & 0 & 0 & 1 & 0 & 0 & 0 & 0 \\
    0 & 0 & 0 & 0 & 1 & 0 & 0 & 0 \\
    0 & 0 & 0 & 0 & 0 & 1 & 0 & 0 \\
    0 & 0 & 0 & 0 & 0 & 0 & 1 & 0 \\
    0 & 0 & 0 & 0 & 0 & 0 & 0 & 1 \\
    \end{bmatrix}
\end{equation*}
and the coherency, $c(\hat{y})$, is defined as 
\begin{equation}
\label{eq:data-coherency}
    c(\hat{y}) = \norm{\hat{y} - A\hat{y}}. 
\end{equation}
The term $A\hat{y}$ properly aggregates the lower-level predictions of $\hat{y}$ to the higher levels of the series. Of course, $A$ can aggregate the series in any other way as well, for example describing how one layer of the hierarchy aggregates to the next higher layer instead. 

The loss $c(\hat{y})$ essentially measures how close $\hat{y}$ is from an output which is truly coherent, namely $A\hat{y}$. Since the ground truth data obeys a known hierarchy (or nearly obeys if there are errors in the data for instance), an accurate forecast should be coherent. Ideally, using this additional known structure of the data should allow us to learn improved forecasts.

\subsection{Network coherency regularization}\label{sec:prop-model}

\begin{figure*}[h]
    \centering
    \includegraphics[scale=0.4]{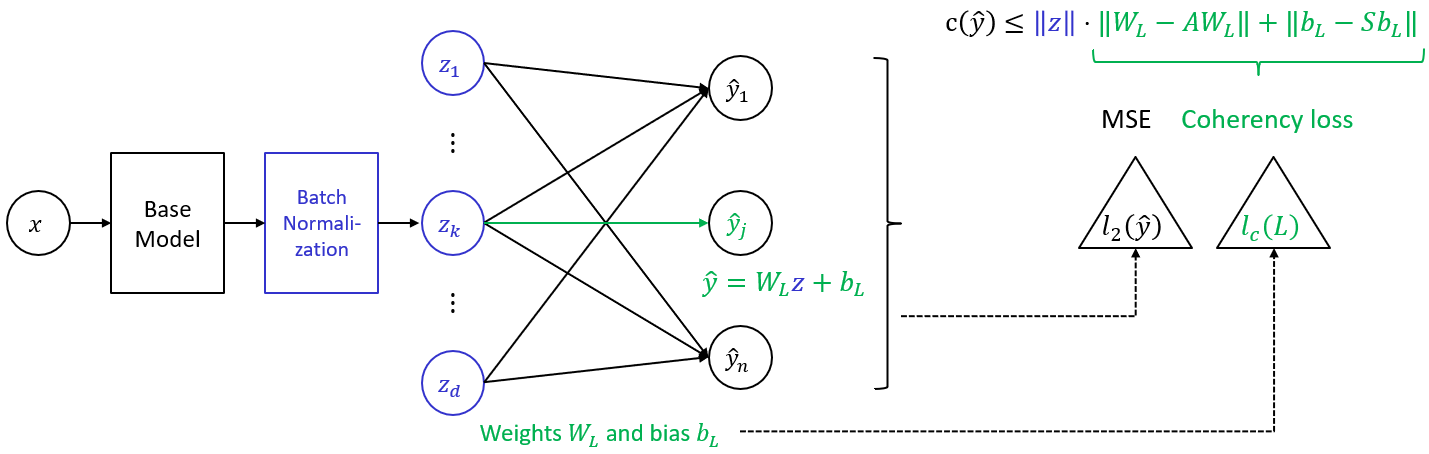}
    \caption{Network and loss architecture. Any neural network architecture followed by final batch normalization. We denote this output by $z$ which is then passed through a final linear layer denoted by $L$ with weights $W_L$ and bias $b_L$. The final output is denoted by $\hat{y}$. We have structural coherency loss ${l}_c(L)$ in addition to traditional mean-squared loss.}
    \label{fig:arch}
\end{figure*}

To control coherency, our goal is to control $c(\hat{y})$ for any forecast $\hat{y}$ made by the model. To do so, we represent the forecast as a function of the final layer of the network and define a regularization function in terms of this final layer instead. This allows us to ensure soft coherency universally as it is now a property of the network itself.

The final output of a network ultimately depends on its last layer, which we will assume is a fully-connected linear layer. This is not generally restrictive as such a final layer is also most common in practice. Let $L$ denote this last layer, and let $W_L$ be the weights and $b_L$ the bias of the layer. Let ${z}$ denote the output of the second-to-last layer, which is also the input to the last linear layer $L$. See figure \ref{fig:arch}. This also allows for flexibility in choosing any architecture to be plugged in for the core architecture. The same approach is applicable to any kind of neural net time series forecasting model, e.g., CNN, Transformer, RNN and so on.  In practice, the initial neural network could be any network that outputs a quantity that can be passed into the final linear layer, since the coherency computation only depends on this final linear layer $L$.

Recall from \eqref{eq:data-coherency} that the \textit{coherency} of the model will be $\begin{Vmatrix} \hat{y} - A \hat{y} \end{Vmatrix}$, the difference between our predictions for each level and our predictions for the aggregate for that level. We define our network coherency regularization as follows.
\begin{definition}[network coherency regularization]
\label{def:coherency-loss}
    Given a final linear layer $L$ with weights $W_L$ and bias $b_L$, we define the network coherency regularization $l_c(L)$ as 
    \begin{equation}
    \label{eq:c-loss}
    l_c(L) = \norm{W_L-AW_L} + \norm{b_L-Ab_L}.
    \end{equation}
\end{definition}
This comes naturally from the definition of coherency. In particular, the coherency of output $\hat{y}$ is defined by 
\begin{align}
    c(\hat{y}) =&\ \norm{\hat{y} - A\hat{y}} \label{eq1-coh}\\
        =&\ \norm{W_L z +b_L - A(W_L z +b_L)}\\
        \leq&\ \norm{z} \norm{W_L - AW_L} + \norm{b_L - Ab_L}.
\end{align}
\paragraph{Practical take-aways.}
The bound on coherency depends on $\norm{z}$, the norm of the output of the next-to-last layer. In practice, we apply a batch-normalization step to the output of the next-to-last layer in order to bound $\norm{z}$. See Proposition \ref{thm} for a theoretical justification. See Figure \ref{fig:arch} for an illustration of the network architecture.

So, minimizing $l_c(L)$ also provides an upper bound on coherency loss $c(\hat{y})$. Our proposed regularization function \eqref{eq:c-loss} is independent of the data, it measures an inherent and structural property of the network itself. As an example, note that when $l_c(L) = 0$, the network is guaranteed to be completely coherent for all inputs. In general, it can be used to describe the coherency of its predictions on any input, including out-of-sample data. Again, soft coherency can be especially useful for cases where the data itself isn't coherent. 

\begin{prop}
\label{thm}
    Given a batch-normalization layer is applied directly before the final layer $ L $, then the network coherency regularization \eqref{def:coherency-loss} bounds the coherency $c(\hat{y})$ (defined in \eqref{eq:data-coherency}) of any output prediction $\hat{y}$ as follows. For any $\delta \geq 1$, the coherency of \emph{any} prediction is bounded as
    \begin{equation}
        c(\hat{y}) \leq \delta \cdot l_c(L)
    \end{equation}
    with probability $        \mathbb{P}\biggl(c(\hat{y}) \leq \delta \cdot l_c(L)\biggr) \geq 1 - 4\exp\paren{-\frac{\delta^2}{8d^2}}$
    where $d$ is the dimension of the second-to-last output ${z}$.
\end{prop}
\noindent The lower the coherency regularization term, the higher the certainty of coherency will be. As the second-to-last network layer has higher width, we need to decrease the network coherency regularization to reach the same coherency $c(\hat{y})$. Practically, this means we should weight the network coherency regularization more. 

\paragraph{Loss weighting} The network coherency regularization  enforces coherency, but not necessarily accuracy, on the network. As such, we combine our network coherency regularization with an accuracy-based metric, $\ell_2(\hat{y})$ (which could be mean-squared error or mean absolute error, for example). In this work we consider the simple weighting $w \cdot l_c(L) + \ell_2(\hat{y})$. The $w$ term is meant to describe the relative weight of the coherency regularization metric. If we set $w=1$, then the two losses are equally important. Tuning $w$ lets us trade off between accuracy and coherency. Soft coherency can lead to a better solution overall compared to restricting the learning search to fully coherent solutions. This is because when enforcing full coherency, the training can become stuck finding solutions that are coherent but becomes more difficult to focus on improving accuracy. Soft coherency may lead to a better optimization problem. We discuss this further in the experiment section.

\begin{proof}[Proof of Proposition \ref{thm}]
    As in \eqref{eq1-coh}, we can rewrite the output $\hat{y}$ of the network in terms of the output of the second to last layer ${z}$. Specifically, $\hat{y} = W_L{z} + b_L$. The coherency can be rewritten as 
    \begin{align}
        c(\hat{y}) =&\ \norm{\hat{y} - A\hat{y}} 
        = \norm{(W_L{z} + b_L)- A(W_L {z} + b_L)} \\ 
        \leq&\ \norm{{z}} \norm{W_L - AW_L} + \norm{b_L - Ab_L}         \leq \norm{{z}} l_c(L)
    \end{align}
    Due to batch-normalization, each term of ${z}$ has mean zero and unit variance. From \cite{ledoux2013probability} (see lemma 3.1), we have $\mathbb{P}(\norm{{z}} \geq \delta) \leq 4\exp\paren{-\delta^2/8\mathbb{E}\norm{{z}}^2} $. Moreover, $\mathbb{E}\norm{{z}}^2 = d$, the dimension of ${z}$.
\end{proof}

\subsection{Distributional forecasts}
\label{sec:dist}

This section describes a method of designing neural network architectures for making a distributional forecast for each time series while ensuring that each sample from this distribution satisfies (soft) coherency. The general approach from existing literature cannot be directly applied to our setting. To see why, consider existing methods from \cite{kamarthi2022profhit}, \cite{rangapuram2021end}. First, these methods predict the mean and variance of a Gaussian distribution (or other sufficient statistics if one uses a different distribution family). The projection method generates samples from this distribution and projects them onto the coherency constraints. The PROFHiT method instead defines a soft coherency metric on this distribution directly. However, our network coherency regularization function is defined on the final layer of the network and hence neither of these two approaches can be directly extended to ours. 

Here we will describe in short several existing methods that can be used to generate distributional forecasts which can be extended to use our network coherency regularization approach. In general, our approach can be used whenever the model is designed to directly generate samples from the distribution, and not by making parametric distributional forecasts (such as predicting mean and variance for Gaussian distributions).  

\paragraph{Variational Autoencoder (VAE) methods}. See for example work of \cite{nguyen2021temporal} for constructing distributional forecasts using VAEs. The model applied to our setting could look as follows. See Figure \ref{fig:vae_arch} for an illustration. (1) As input, we take features (which can be lag features from the time series in addition to any other available features). (2) We encode these features to a latent space embedding $\theta$ of lower dimension. (3) We add random normal noise to $\theta$  and finally (4) decode this noisy embedding vector to produce a final prediction. (5) Sampling multiple such predictions by using multiple noise values added to $\theta$ gives us the empirical forecast distribution. The network coherency regularization can be applied to the final layer of the decoding network. 
    
\paragraph{Dropout layers}. Dropout is used in many models in deep learning as a way to avoid over-fitting, however, \cite{gal2016dropout} also shows it also allows one to represent distributional forecasts. We can extend the idea to our setting as follows. Droupout can be applied to any number of layers of the of the model (not including the final layer). Multiple forward passes through the network using the same input will generate a distribution of samples (since at each forward pass, a random subset of nodes will be ``dropped out''). Next, we can enforce coherency by applying the network coherency regularization \eqref{def:coherency-loss} to the final layer of the model (which we intentionally do not apply dropout to).

\paragraph{Other sample-generating generative models for time series forecasting.} In general, our regularization method can be applied to any forecasting method which directly generates samples. As other examples, we can use generative adversarial methods \cite{vuletic2024fin}, normalizing flow for forecasting \cite{rasul2020multivariate} (also mapping random noise to output distribution predictions) or diffusion methods \cite{rasul2021autoregressive}.

\begin{figure}[H]
    \centering
\includegraphics[width=0.75\linewidth]{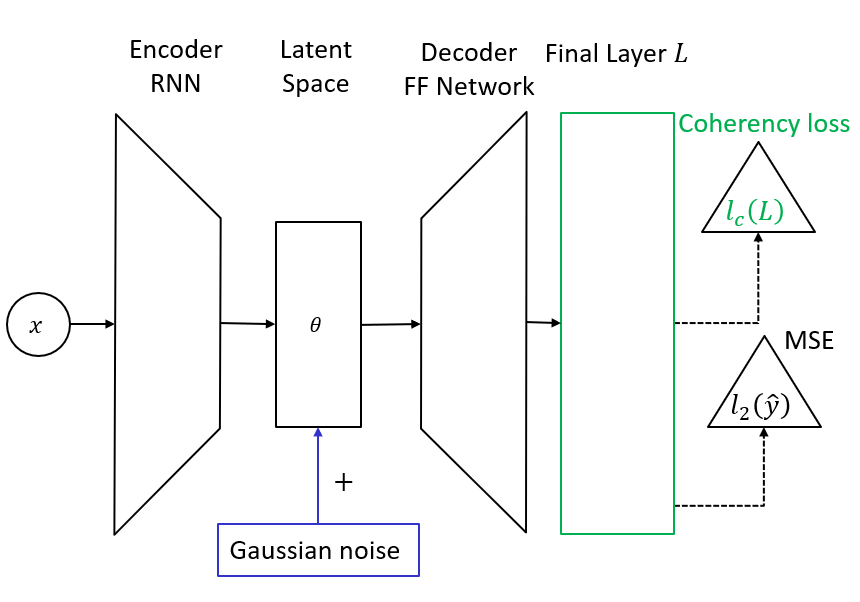}
    \caption{VAE-based distributional forecasting for network coherency regularization.}
    \label{fig:vae_arch}
\end{figure}

The advantage of constructing non-parametric distributions like we do is (1) the ability to ensure every sample from the distribution is (softly) coherent and (2) not making any constraining assumption on the underlying predictive distribution. If one assumes a parametric form (like Gaussian), it may likely not be accurate.  With this more flexible form for the distribution, we can more closely match the true data distribution. In contrast to our approach, other soft coherency methods like \cite{kamarthi2022profhit} cannot be extended to use these architectures and methods whose output is a single sample from the predicted distribution rather than the entire distribution. This is because the soft coherency loss (distributional consistency error from \cite{kamarthi2022profhit}) is defined on entire distributions such as Gaussian distributions, not on single sample outputs. It may be possible to extend \cite{kamarthi2022profhit} to this setting, although that is beyond the scope of this paper. 

\section{Computational results}

All code and data can be found in the anonymized github: \url{https://github.com/CoRe-Hierarchical-time-series/KDD-2025}.

\subsection{Datasets}

We test on three benchmark datasets, which are also used in \citet{rangapuram2021end} and \citet{kamarthi2022profhit}: Traffic \cite{cuturi2011fast}, Labor \cite{abs2020labourforce}, and Tourism (Large) \cite{kourentzes2019cross}. Table \ref{tab:datasets} shows the total number of series (including the entire hierarchy), hierarchy levels, and observations of each dataset. Finally, the prediction horizon is 1 for each dataset.

\paragraph{Noisy data} These datasets are perfectly coherent, so we will also generate noisy datasets based off of these to show the advantages of soft coherency in this setting. For each dataset, we construct a new non-coherent dataset by removing some time series at random. This is meant to simulate random missing series from the data.  To construct noisy data, we use the original datasets from Table \ref{tab:datasets} and randomly drop different time series. In particular, we construct 20 random datasets, each one constructed by dropping a different random subset of $20\%$ of the bottom-level time series. 

\begin{table}[t]
    \begin{center}
    \begin{tabular}{c c c c} \toprule
      Dataset & \# of series & \# of levels & \# of observations \\ [0.5ex] \midrule
     Traffic & 207 & 4 & 366 \\ 
     Labour & 57 & 4 & 514 \\ 
     Tourism & 555 & 21 & 228 \\ \bottomrule
    \end{tabular}
    \caption{Number of series, hierarchy levels, and observations for each dataset. 
    }
    \label{tab:datasets}
    \end{center} 
\end{table}

\subsection{Metrics}

We measure three metrics for each approach: (1) mean squared error (MSE) on test data at each level of the hierarchy, (2) weighted mean absolute error (WMAPE) by hierarchy level, and (3) coherency as defined by  \eqref{eq:data-coherency}. WMAPE is defined by
\begin{equation}
    \text{WMAPE} = \frac{\sum_{t=1}^T\sum_{n=1}^N |\hat{y}_n^t - y_n^t|}{\sum_{t=1}^T\sum_{n=1}^N y_n^t}
\end{equation}
where $\hat{y}_n^t$ are the predicted values for the $n^{th}$ series at time $t$ and $y_n^t$ is the corresponding ground truth. This is a commonly used metric for regression problems which intuitively measures relative error in forecasting weighted by  the value of each series.

For the distributional forecasts, we also measure the Cumulative Ranked Probability Score (CRPS) which is a widely used standard metric for the evaluation of probabilistic forecasts that measures both accuracy and calibration. This is defined as follows. Given a ground truth observation $y$ (a point value) and the predicted probability distribution $\hat{p}$, let $\hat{F}$ denote its cumulative distribution function. Then, 
\begin{equation}
    \text{CRPS} = \int_{-\infty}^{\infty} (\hat{F}(\hat{y}) - \mathbbm{1}(\hat{y} > y))^2 d\hat{y}.
\end{equation}
In practice, we evaluate CRPS by sampling from the predicted probability distribution and approximating the integral with these samples. 

\subsection{Models}
\label{sec:models}

We describe each model we use and compare against. Again, any possible network architecture can be used up to the final linear layer. For the experiments we make one uniform choice that works well for a variety of datasets, specifically an RNN architecture as the baseline. 

\begin{enumerate}[leftmargin=*]
    \item \emph{Baseline}. This is a simple RNN with a final output linear output layer to predict all time series simultaneously. This baseline approach uses only the mean-squared error loss. We only vary the coherency computations in the alternatives described below so that we can directly compare the training and results of our specific coherency learning method. Specific architecture choices for the RNN are described in the appendix \ref{app:model}.
    \item \emph{Network coherency regularization (CoRe)}. Our proposed approach will use the same base RNN model, and only alter the loss function to use our proposed coherency loss in \eqref{def:coherency-loss}. We choose the parameter $w$ for the tradeoff between network coherency regularization and MSE by validation on a holdout set.
    \item \emph{End-to-end projection}. We also compare against the end-to-end projection method of \citet{rangapuram2021end} by projecting $\hat{y}$ into coherent space and using the $\ell_2$ loss.  The projection model provides us with a hard coherency constraint model to compare against. This method is already state-of-the art, as can be seen through their experiments on the same datasets (as some of ours) comparing against existing methods. Hence, this is our primary strictly coherent benchmark. 
    \item \emph{PROFHiT} Finally, we compare against another soft coherency method from \cite{kamarthi2022profhit}. Similar to our approach, this method also has a hyperparameter for weighting between MSE and their coherency-based loss. We choose this parameter in the same way by validation on the same holdout set as for our approach. 
\end{enumerate}

\noindent The only hyperparameters we need to tune are the weights in the loss function between MSE and the coherency term (either our proposed regularization term or the proposed loss from \cite{kamarthi2022profhit}). For each dataset, we split into train, validation and test. Train consists of the first 80\% of the data (where data is ordered by time), validation consists of the next 10\% and test the final 10\%.  We train each method until we observe no improvement in validation score (of MSE) for at least 100 epochs. We then choose the checkpoint during training which had the best performance on validation data. Finally, we evaluate each model on the test data. Further, we choose the hyperparameter in weighting the loss between coherency by hold-out validation. We evaluate a range of weights ($1e-4, 1e-3, 1e-2, 1e-1$) and choose the best performance  based on the validation data.

\subsection{Results}

  

\begin{figure}[t]
  \vspace{0pt} 
  \begin{minipage}[t]{0.49\columnwidth}
    \centering
    \includegraphics[width=1\columnwidth]{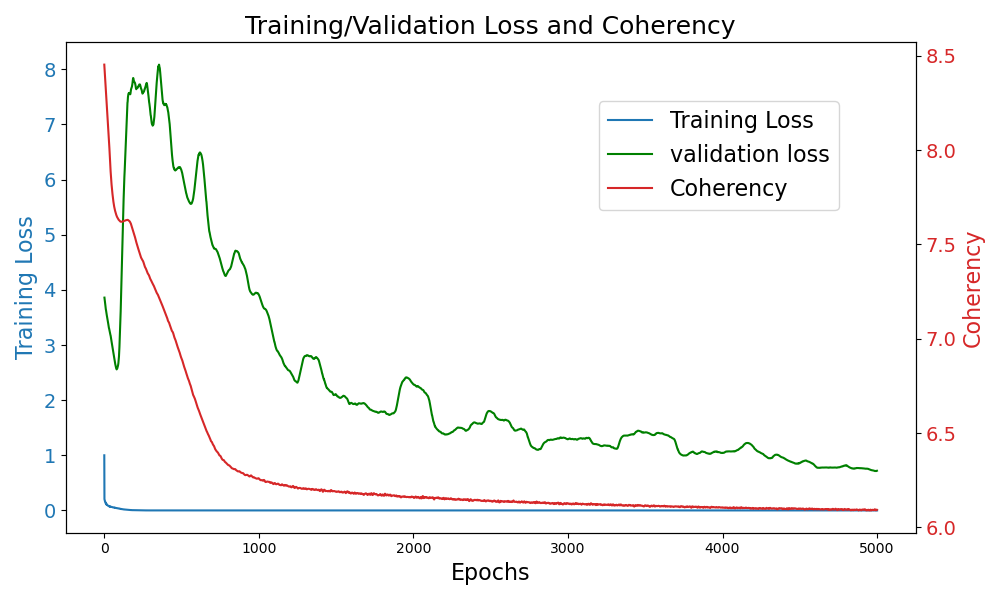}
    \captionof{figure}{Validation score follows trend of coherency metric rather than training score, indicating coherency is crucial for generalization.}
    \label{fig:validation-coherency}
  \end{minipage}
  \begin{minipage}[t]{0.49\columnwidth}
    \centering
    \includegraphics[width=0.9\columnwidth]{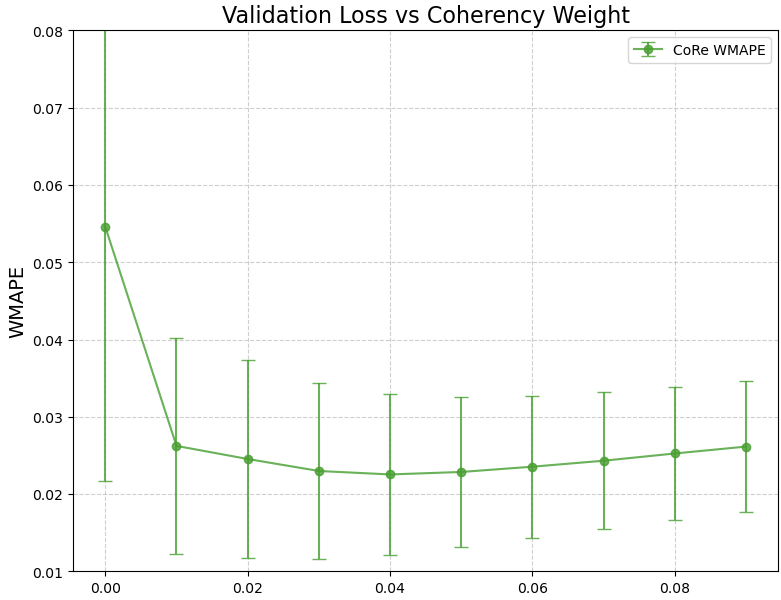}
    \captionof{figure}{Validation score as coherency weight increases.}
    \label{fig:validation-coherency-weight}
  \end{minipage}
\end{figure}

We split the results into two sections. First, point forecasts and second, distributional forecasts.  


\paragraph{Coherency improves generalization} First, we analyze the behavior of the network coherency regularization during training and effect on validation loss. In figure \ref{fig:validation-coherency} we plot the progression of training loss, validation loss, and the network coherency regularization term of the model trained on the labor dataset. We observe that training loss decreases rapidly to nearly zero, suggesting overfitting which is corroborated by a sharp increase in validation loss in early training epochs. However, network coherency regularization term decreases as training progresses. Correcting the model to output more coherent predictions improves the validation loss and generalization of the model. 

This is also evidenced when using the other coherency methods. Table \ref{tab:no-noise} shows both average coherency and average WMAPE across 10 training runs on each dataset. We observed the test loss improved for all coherency methods over a base model with no coherency.

\paragraph{Tradeoff in weighting network coherency regularization with MSE} Recall our training loss is a weighted combination of MSE and our proposed network coherency regularization: $w \cdot l_c(L) + \ell_2(\hat{y})$. As we increase the weight $w$ of the network coherency regularization, coherency improves, but also makes it more difficult to improve model accuracy $\ell_2(\hat{y})$. In Figure \ref{fig:validation-coherency-weight} we plot the validation loss for a variety of such weights. For each weight, we train 10 separate models and plot the average final validation loss. For consistency, we use the same starting seed/initialization for each choice of $w$. Applying no coherency whatsoever ($w=0$) provides the worst out-of-sample results. As expected, we observe that increasing $w$ from zero provides better generalization, but past a certain point (in this case, after $w=0.04$) generalization worsens. Too heavy of a focus on coherency during training can possible cause the model to become stuck prioritizing coherency over predictive performance.

\begin{table*}[t]
    \centering
\begin{tabular}{ c c c c c c}
\toprule
Model & \multicolumn{4}{c}{{Hierarchy level}} & Average MSE \\
 & 4 & 3 & 2 & 1 \\ 
\midrule
Base & 0.0053 / 15.8160 & 0.0226 / 0.5279 & 0.0616 / 0.4013 & \textbf{0.0997} / 0.2586 & 0.047 \\ 
CoRe (Ours) & 0.0167 / 19.1570 & \textbf{0.0147} / \textbf{0.3296} & 0.0311 / \textbf{0.2576} & 0.1102 / \textbf{0.2484} & \textbf{0.043} \\ 
PROFHiT & 0.0412 / 45.7371 & 0.0187 / 0.5084 & \textbf{0.0275} / 0.3544 & 0.1255 / 0.3712 & 0.053 \\ 
Projection & \textbf{0.0050} / \textbf{15.5177} & 0.0156 / 0.5069 & 0.0516 / 0.4515 & 0.1990 / 0.4432 & 0.068 \\ 
\bottomrule
\end{tabular}
    \caption{Noisy traffic data results; MSE / WMAPE (respectively) at each level of the hierarchy. Bottom hierarchy level corresponds to level 4 and top level aggregation corresponds to level 1.}
    \label{tab:traffic-res}
\end{table*}

\begin{table*}[t]
    \centering
\begin{tabular}{ c c c c c c}
\toprule
Model & \multicolumn{4}{c}{{Hierarchy level}} & Average MSE \\
 & 4 & 3 & 2 & 1 \\ 
\midrule
Base & 0.0053 / 2.5026 & 0.0059 / 0.9327 & 0.0076 / 0.5429 & 0.1263 / 0.3392 & 0.036 \\ 
CoRe (Ours) & \textbf{0.0010} / \textbf{1.0988} & \textbf{0.0016} / \textbf{0.5072} & \textbf{0.0042} / \textbf{0.4210} &\textbf{ 0.1118} / 0.3510 & \textbf{0.029} \\ 
PROFHiT & 0.0021 / 1.5435 & 0.0025 / 0.6193 & 0.0048 / 0.4332 & 0.1158 / \textbf{0.3310} & 0.031 \\ 
Projection & 0.0033 / 1.9587 & 0.0031 / 0.6910 & 0.0058 / 0.4941 & 0.1190 / 0.3408 & 0.033 \\ 
\bottomrule
\end{tabular}
    \caption{Noisy labor data results; MSE /  WMAPE (respectively) at each level of the hierarchy. Bottom hierarchy level corresponds to level 4 and top level aggregation corresponds to level 1.}
    \label{tab:labor-res}
\end{table*}



\begin{figure*}[h]
    \centering
    \includegraphics[width=0.9\linewidth]{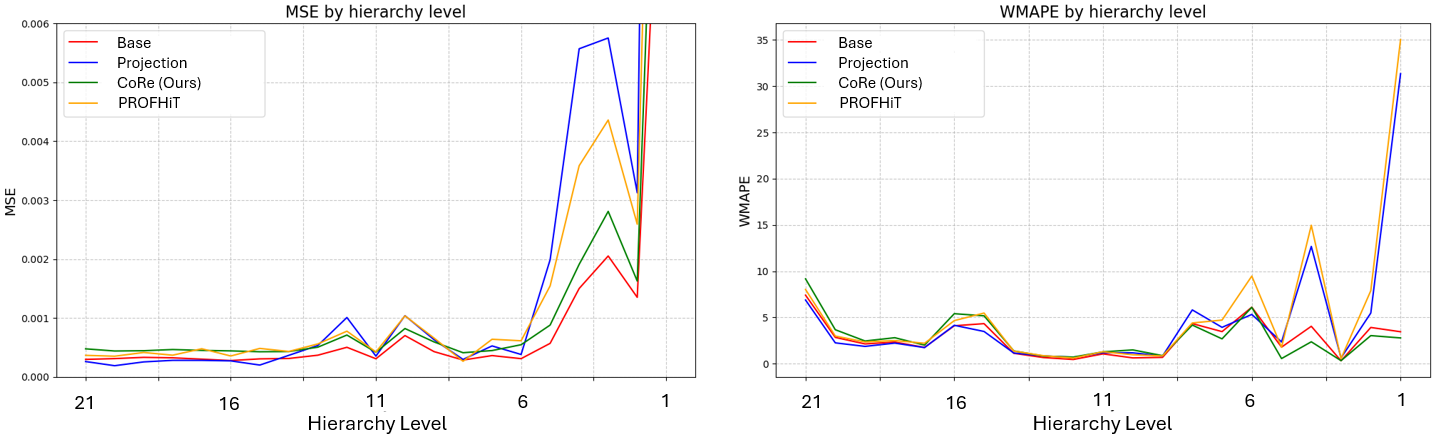}
    \caption{Noisy tourism data results. Average MSE (left) and WMAPE (right) at each level of the hierarchy across 20 noisy datasets.}
    \label{fig:tourism-res}
\end{figure*}

\begin{table}[H]
    \centering
    \begin{tabular}{c c c c c} \toprule
& Base & \makecell{CoRe (Ours)} & PROFHiT & Projection  \\ \midrule     
    Traffic &  0.037 / 0.039 & 0.007 / 0.038 & 0.034 / 0.06 & 0 / 0.037 \\
    Labor & 0.09 / 0.45 & 0.03 / 0.4 & 0.04 / 0.42 & 0 / 0.40    \\
    Tourism & 0.065 / 1.1 & 0.004 / 0.37 & 0.014 / 0.44 & 0 / 0.80 \\ \bottomrule        
    \end{tabular}
    \caption{Average (across 10 training runs) of coherency / WMAPE for each method on test data from each dataset.}
\label{tab:no-noise}
\end{table}

\paragraph{Results on fully-coherent data} Again, Table \ref{tab:no-noise} shows results for each model on each original dataset across 10 training runs. We primarily present these results to show each method is roughly comparable in this setting. Overall, our approach is more coherent then the PROFHiT methodology while reaching the same or better WMAPE results. The  exact projection methods is of course more coherent, although this does not necessarily result in better WMAPE performance. It depends on the specific use-case whether one wants to choose a perfectly coherent method or a soft-coherency method. 



\paragraph{Results on noisy data}

Recall that to construct noisy data, we use the original datasets from Table \ref{tab:datasets} and randomly drop 20\% of bottom-level time series.  We report the average metrics (i.e., MSE, WMAPE, and coherency on test data) at each level of the hierarchy across all 20 noisy datasets for each of the three base datasets from table \ref{tab:datasets}. See Table \ref{tab:traffic-res} for MSE/WMAPE results on the traffic dataset, Figure \ref{fig:tourism-res} for the tourism data, and Table \ref{tab:labor-res} for the labor data. Table \ref{tab:coh-results} contains all coherency results. 

On noisy data, the base model performs best in terms of accuracy overall when there is more data. This is evidenced by the tourism results which has the largest number of time series. However, the base model is unaware of any hierarchical constraints (and must learn this purely from data). Therefore, coherency levels of the base model are quite poor across all datasets. Coherency results for each method across all datasets can be found in Table \ref{tab:coh-results}. These results also corroborate our discussion earlier around Figure \ref{fig:validation-coherency} that not enforcing forecast coherency can lead to worse generalization capabilities which can be better seen for the labor and traffic datasets (Tables \ref{tab:traffic-res} and \ref{tab:labor-res}) which have less data/series.

Overall, our proposed network coherency regularization approach can achieve the best coherency on test while still performing as well or better in terms of accuracy. Especially in this setting when the data is not exactly coherent, there is a larger gap in performance compared to the exact projection method. The PROFHiT method also generally shows improvement over the exact projection approach for the same reason, but does not achieve as good performance as our method.

\begin{table}
    \centering
    \begin{tabular}{c c c c c} \toprule
& Base & \makecell{CoRe (Ours)} & PROFHiT & Projection  \\ \midrule     
    Traffic &  0.45 $\pm$ 0.32 & \textbf{0.01} $\pm$ 0.008 & 0.1 $\pm$ 0.28 & 0 $\pm$ 0 \\
    Labor   & 0.18 $\pm$ 13 & \textbf{0.11} $\pm$ 0.10 & 0.14 $\pm$ 0.08 & 0 $\pm$ 0 \\
    Tourism & 0.07 $\pm$ 0.075 & \textbf{0.035} $\pm$ 0.03 & 0.056 $\pm$ 0.04 & 0 $\pm$ 0 \\  \bottomrule        
    \end{tabular}
    \caption{Average coherency and standard deviation of each method on test data from each noisy dataset.}
\label{tab:coh-results}
\end{table}

\paragraph{Distributional forecasts}
\label{sec:dist-forecast}

We implement both methods proposed in Section \ref{sec:dist} to generate distributional forecasts. In particular, the VAE approach follows the architecture in Figure \ref{fig:vae_arch}. The encoder is the same RNN architecture as in the previous section. We use the final hidden state of the RNN as the latent space embedding. As the decoder step, we use two feed-forward layers. On the other hand, the dropout approach described in section \ref{sec:dist} uses the same architecture, although we do not add any noise to the latent space embedding when performing the forward pass. To inject noise, we instead dropout a random subset of the nodes in the hidden layer of the RNN encoder. In both approaches, the network coherency regularization can be applied to the final linear layers of the networks.  In this distributional setting, we also measure CRPS. To evaluate MSE or WMAPE, we use the mean of the forecasted distribution. Appendix Figure \ref{fig:crps_weight} shows CRPS on validation data as a function of the weight. Similar to results in Figure \ref{fig:validation-coherency-weight} for the point forecast version, we observe the tradeoff between increasing coherency weight and accuracy, CRPS in this case.

\begin{table*}[h]
    \centering
    \resizebox{16cm}{!}{
    \begin{tabular}{c c c c c c c | c c c c c c}
        & \multicolumn{6}{c}{Dropout} & \multicolumn{6}{c}{VAE}\\
        \toprule 
        Dataset & \multicolumn{3}{c}{{CoRe (Ours)}} & \multicolumn{3}{c}{{\makecell{Base}}} & \multicolumn{3}{c}{{CoRe (Ours)}} & \multicolumn{3}{c}{{\makecell{Base}}} \\ 
        & CRPS & WMAPE & Cohr. & CRPS & WMAPE & Cohr. & CRPS & WMAPE & Cohr. & CRPS & WMAPE & Cohr. \\
        \midrule
        Traffic & \makecell{\textbf{0.006} \\ $\pm$ \textbf{0.004}} & \makecell{\textbf{0.07} \\ $\pm$ \textbf{0.08}} & \makecell{\textbf{0.01} \\ $\pm$ 0.013}  & \makecell{0.008 \\ $\pm$ 0.006} & \makecell{0.10 \\ $\pm$ 0.12 } & \makecell{\textbf{0.01} \\ $\pm$ 0.005} & \makecell{\textbf{0.01} \\ $\pm$ \textbf{0.007} } & \makecell{\textbf{0.16} \\ $\pm$ {0.24}} & \makecell{\textbf{0.008} \\ $\pm$ \textbf{0.009}} & \makecell{0.025 \\ $\pm$ 0.20} & \makecell{0.22 \\ $\pm$ \textbf{0.23}} & \makecell{0.017 \\ $\pm$ 0.011}\\ 
        Labor & \makecell{ 0.068 \\ $\pm$ \textbf{0.092} } & \makecell{\textbf{0.28} \\ $\pm$ \textbf{0.03}} & \makecell{\textbf{0.005} \\ $\pm$ \textbf{0.009} } & \makecell{ \textbf{0.065} \\ $\pm$ 0.01 } & \makecell{ 0.32 \\ $\pm$ 0.11} & \makecell{ 0.03 \\ $\pm$ 0.02 } & \makecell{\textbf{0.017} \\ $\pm$ \textbf{0.03}} & \makecell{\textbf{0.43} \\ $\pm$ \textbf{0.20}} & \makecell{\textbf{0.0016} \\ $\pm$ \textbf{0.0022}} & \makecell{0.020 \\ $\pm$ 0.032} & \makecell{0.73 \\ $\pm$ 0.63} & \makecell{0.031 \\ $\pm$ 0.026} \\ 
        Tourism & \makecell{\textbf{0.014} \\ $\pm$ \textbf{0.025}} & \makecell{\textbf{0.38} \\ $\pm$ \textbf{0.17}} & \makecell{\textbf{0.004} \\ $\pm$ \textbf{0.003}} & \makecell{0.017 \\ $\pm$ 0.029} & \makecell{0.64 \\ $\pm$ 0.51} & \makecell{0.022 \\ $\pm$ 0.020} & \makecell{\textbf{0.01} \\ $\pm$ \textbf{0.017}} & \makecell{\textbf{0.33} \\ $\pm$ \textbf{0.20}} & \makecell{\textbf{0.005} \\ $\pm$ \textbf{0.003}} & \makecell{0.01 \\ $\pm$ 0.016} & \makecell{0.41 \\ $\pm$ 0.31} & \makecell{ 0.01 \\ $\pm$ 0.013 } \\ \bottomrule
    \end{tabular}
    }
    \caption{Metrics for distributional forecasts for the base model without coherency, and our proposed regularization approach. }
    \label{tab:dist-no-noise}
\end{table*}

\begin{figure*}
    \centering \includegraphics[width=0.9\linewidth]{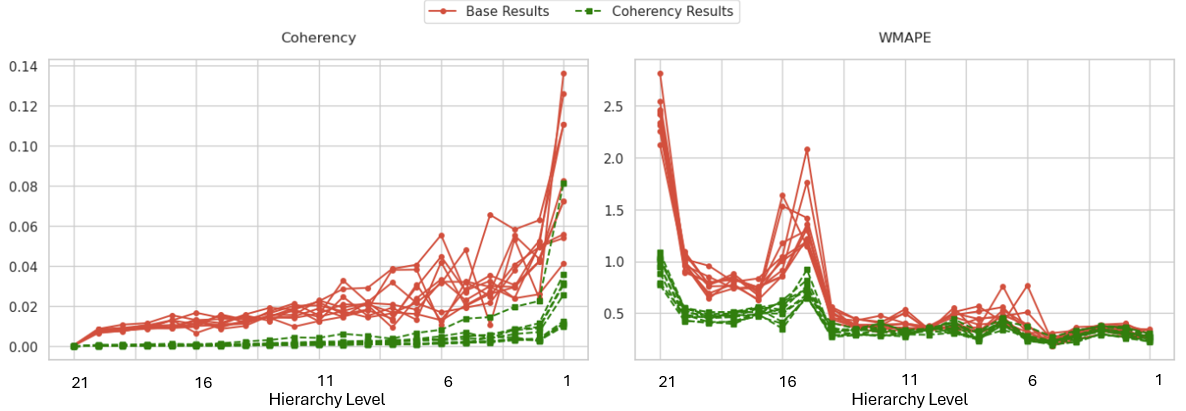}
    \caption{Coherency reduces variance. Results on tourism dataset using dropout method for distributional forecasts. Level 1 corresponds to the top level (aggregating all series), while level 21 corresponds to the bottom level series.}
    \label{fig:variance}
\end{figure*}

\paragraph{Coherency reduces variance in training.} 

We consider training the same model architecture both with our network coherency regularization and without. Figure \ref{fig:variance} shows results for the tourism dataset on validation data for both approaches across 10 different training runs. We observe not only improved performance both in terms of coherency, CRPS, and WMAPE, but also a reduction in variance across each run when training with our proposed regularization. Table \ref{tab:dist-no-noise} shows aggregate results for all three datasets. 

\paragraph{Comparison with PROFHiT} We observe that PROFHiT has difficulty in improving both accuracy (lowering CRPS, WMAPE, MSE) while simultaneously improving coherency in the distributional forecast setting. See for instance table \ref{tab:profhit-dist} showing three models with different training weight between MSE and their proposed coherency loss. We expect this to be due to a misalignment in the structure of distributions forecasted by PROFHiT and the ground truth data distribution. PROFHiT produces a Gaussian distributional forecast for each individual series. However, creating one sample for each series will result in point forecasts which are not coherent and there is no explicit control over how coherent the final sample will be. The loss function proposed in \cite{kamarthi2022profhit} appears to improve coherency by (1) ensuring the mean of the distributional forecasts for all series is coherent and (2) reduces variance for each individual Gaussian distribution forecast. This is evidenced in Figure \ref{fig:profhit-var} in the appendix, showing that higher coherency leads to distributional forecasts will lower variance. However, this is not the true distribution of the data. While a single series may exhibit a Gaussian distribution, each series is not independent from the rest, causing the discrepancy with respect to the PROFHiT model. To contrast, a method which generates samples from the forecasted distribution can ensure individual samples are coherent. We see from Table \ref{tab:dist-no-noise} that such approaches can perform up to twice as well in terms of CRPS, while maintaining nearly the same coherency.

\begin{table}[t]
    \centering
    \begin{tabular}{l c c}
        \toprule
        & CRPS & Coherency Loss \\ \midrule
        low coherency & 0.031 & 0.032 \\
        medium coherency & 0.050 & 0.013 \\ 
        high coherency & 0.062 & 0.001 \\ \bottomrule
    \end{tabular}
    \caption{PROFHiT distributional forecasts on tourism data. }
    \label{tab:profhit-dist}
\end{table}

\noindent \paragraph{\textbf{Conclusion}} We introduce CoRe, a coherency regularization method to impose soft coherency on hierarchical time series problems. Our approach is generalizable to forecasting distributional forecasts where each sample is (softly) coherent. We prove our approach ensures soft coherency on any output it generates, even out-of-sample and out-of-distribution. We empirically show strong performance for a variety of datasets and network architectures.

\newpage 
\section{Citations and Bibliographies}

\bibliographystyle{ACM-Reference-Format}
\bibliography{sample-base}

\appendix

\newpage 
\section{Experiment details}

\paragraph{Data processing}

We apply standard preprocessing to each dataset by converting each value between 0 and 1. We do so by dividing each entry by the largest value observed in the entire data. Note that we scale each datapoint by the same value (not, as is common, by the maximum value in the same column/feature). We do so in order to keep the hierarchical nature of the data. 

Since we work with time series data, we augment each datapoint with lag features as well. That is, a single datapoint for time point $t$ consists of all series in the hierarchy at times $t, t-1, \dots, t-k+1$. For all experiments, we chose a window of $k=5$.  

\paragraph{Tuning}

For each dataset, we split into train, validation and test. Train consists of the first 80\% of the data (where data is ordered by time), validation consists of the next 10\% and test the final 10\%.  We train each method until we observe no improvement in validation score (of MSE) for at least 100 epochs. We then choose the checkpoint during training which had the best performance on validation data. Finally, we evaluate each model on the test data. 

The only hyperparameters we need to tune are the weights in the loss function between MSE and the coherency term (either our proposed regularization term or the proposed loss from \cite{kamarthi2022profhit}). Otherwise, all other parameters are kept constant --- learning rate of 0.001, batch size equal to the entire training set.

\paragraph{Model architecture}
\label{app:model}

We use the same base network for each experiment: a recurrent neural network (RNN) where we use a hidden layer width of 128. In particular, the model looks as follows. We are given an input sequence with 5 time points $y^1, \dots, y^5$ where each vector $y^i$ consists of all series at each level of the hierarchy. The RNN transforms each time step by an input layer, mapping it to dimension 128. The model maintains a hidden state (also of dimension 128) that is updated at each time step by combining the current input with the previous hidden state. This combination is passed through a hidden layer with a tanh activation function. After processing the entire sequence, the final hidden state undergoes batch normalization. The last layer then transforms this normalized hidden state to produce the output.


\newpage 
\section{Additional Figures}

\begin{figure}[H]
    \centering
    \includegraphics[width=0.95\linewidth]{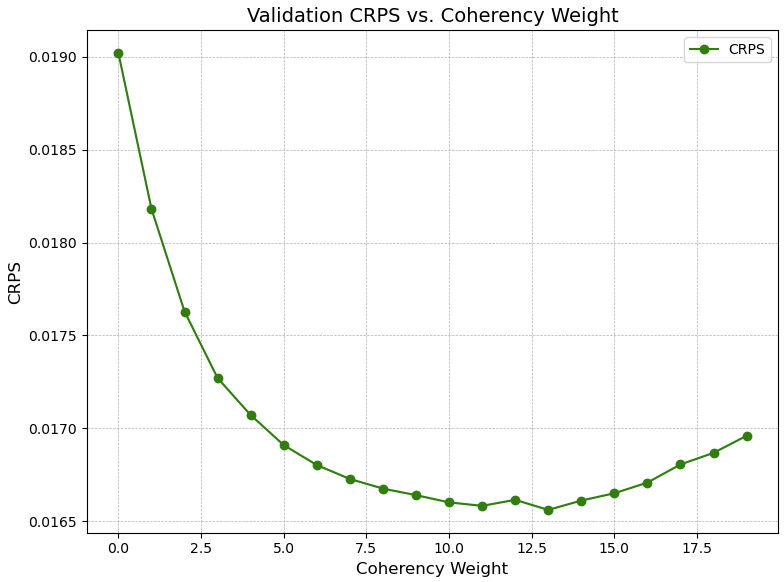}
    \caption{CRPS score on validation data for CoRe model on labor data as a function of regularization weight.}
    \label{fig:crps_weight}
\end{figure}

\begin{figure}[H]
    \centering
    \includegraphics[width=0.95\linewidth]{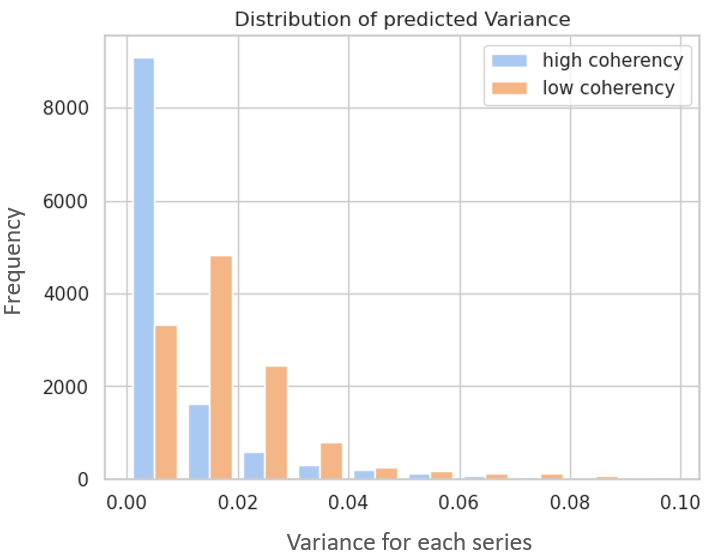}
    \caption{Frequence of variance of predicted gaussian distributions made by PROFHiT method on tourism data. For a model with higher coherency (as enforced by coherency loss introduced in \cite{kamarthi2022profhit}), variance is lower to allow for better coherence.}
    \label{fig:profhit-var}
\end{figure}


\end{document}